\documentclass[conference]{IEEEtran}
\usepackage{amsmath,amssymb,amsfonts}
\usepackage{algorithmic}
\usepackage{graphicx}
\usepackage{textcomp}
\usepackage{xcolor}
\usepackage{times}
\usepackage{enumerate}
\usepackage{amsmath,amssymb,amsfonts}
\usepackage{mathtools} 
\usepackage{amsthm} 

\usepackage{cleveref}
\usepackage{mathrsfs}
\usepackage{comment}
\usepackage{cite}

\def\BibTeX{{\rm B\kern-.05em{\sc i\kern-.025em b}\kern-.08em
    T\kern-.1667em\lower.7ex\hbox{E}\kern-.125emX}}

\usepackage{placeins}
\usepackage{amsthm}
\newtheorem{definition}{Definition}
\newtheorem{proposition}{Proposition}
\newtheorem{theorem}{Theorem}
\newtheorem{corollary}{Corollary}
\newtheorem{remark}{Remark}
\newtheorem{lemma}{Lemma}
\usepackage{geometry}

\ifCLASSINFOpdf
\else
\fi
\IEEEoverridecommandlockouts

\begin{document}
%
\title{Instantiating Bayesian CVaR lower bounds in Interactive Decision Making Problems}
%
%
%

\author{
\IEEEauthorblockN{Raghav Bongole, Tobias J. Oechtering, and Mikael Skoglund}\\
\IEEEauthorblockA{Department of Information Science and Engineering (ISE)\\
KTH Royal Institute of Technology\\
}
\thanks{This work is supported by the Knut and Wallenberg Foundation.}
}

%
%

\markboth{Journal of \LaTeX\ Class Files,~Vol.~13, No.~9, September~2014}%
{Shell \MakeLowercase{\textit{et al.}}: Bare Demo of IEEEtran.cls for Journals}
%



\maketitle


{\begin{abstract}
Recent work established a generalized-Fano framework for lower bounding prior-predictive (Bayesian) CVaR in interactive statistical decision making. In this paper, we show how to instantiate that framework in concrete interactive problems and derive explicit Bayesian CVaR lower bounds from its abstract corollaries. Our approach compares a hard model with a reference model using squared Hellinger distance, and combines a lower bound on a reference hinge term with a bound on the distinguishability of the two models. We apply this approach to canonical examples, including Gaussian bandits, and obtain explicit bounds that make the dependence on key problem parameters transparent. These results show how the generalized-Fano Bayesian CVaR framework can be used as a practical lower-bound tool for interactive learning and risk-sensitive decision making.
\end{abstract}
}


\section{Introduction}
Lower bounds play a central role in statistical decision making by quantifying the smallest loss or cost that any problem must incur. In a Bayesian formulation of statistical decision making, one assumes a prior \(\mu\) on the model class \(\mathcal{M}\) and evaluates performance under the prior-predictive probability. The resulting Bayes risk provides a natural benchmark for what can be achieved on average under the prior, and lower bounds on this quantity reveal information-theoretic limits of learning and decision making. Classical tools such as Le Cam's method, Assouad's lemma, and Fano's inequality are fundamental techniques for deriving such lower bounds in passive statistical estimation \cite{tsybakov2009introduction}.

Many modern learning and control problems are interactive: the learner's actions affect what data are observed. Chen et al.~\cite{chen2024assouad} formalize this through the interactive statistical decision making (ISDM) framework. In ISDM, a model \(M\in\mathcal{M}\) specifies the environment, an algorithm \(\textup{Alg}\in\mathcal{D}\) chooses actions or decisions, and the resulting observation or transcript \(X\) is distributed according to an algorithm-dependent law
\(
X\sim P^{M,\textup{Alg}}.
\)
This algorithm dependence is the key distinction from passive estimation, where the data law is fixed and independent of the estimator. The ISDM viewpoint encompasses passive estimation, multi-armed bandits, and reinforcement learning within a common framework \cite{chen2024assouad,lattimore2020bandit,sutton2018reinforcement}. It is especially relevant in control and sequential decision making, where information acquisition and performance are inherently coupled: one must act in order to learn, and those actions shape both future information and eventual cost. Most lower-bound results in these settings focus on expected loss or expected regret. While expected risk is a natural measure of performance, it can conceal rare but costly failures. This motivates risk-sensitive criteria that emphasize the tail of the loss distribution. A standard example is Conditional Value-at-Risk,
\(
\textup{CVaR}_\alpha(L),
 \alpha\in[0,1),
\)
which measures the expected loss in the worst \((1-\alpha)\)-tail. For losses, \(\textup{CVaR}\) is more sensitive to adverse events than the mean and is widely used in finance, optimization, and risk-sensitive control \cite{rockafellar2000optimization,rockafellar2002general,acerbi2002expected}. Related tail criteria, such as quantiles and Value-at-Risk, have also been studied in statistical learning; see, for example, \cite{ma2024highprobability}. In contrast to Bayes expected-risk lower bounds, Bayesian \(\mathrm{CVaR}\) lower bounds quantify how small the tail-sensitive
performance criterion can be under the prior-predictive probability law, and hence what no algorithm can improve on. A recent step in this direction is the work of Bongole et al.~\cite{bongole2026generalizing}, which extends the interactive Fano method of Chen et al.~\cite{chen2024assouad}. They consider bounded transforms of the loss and derive a generalized interactive Fano inequality for expected bounded transforms under the prior-predictive law. As a consequence, they obtain a general lower-bound template for Bayesian \(\textup{CVaR}\) of bounded losses in ISDM. Their results show that tail-sensitive Bayesian performance criteria can be treated within the same information-theoretic framework as expected-risk and quantile-type criteria.

The main challenge left open by \cite{bongole2026generalizing} is instantiation. The corollaries identify the quantities: a benchmark term under a reference law and a divergence budget, but applying them in a concrete problem still requires a problem-specific calculation. In this work, the objective is to make this framework usable in concrete settings. To this end, a reusable two-point Bayesian \(\textup{CVaR}\) template is extracted from the generalized-Fano corollaries and then instantiated in both a passive Gaussian mean estimation problem and an interactive two-armed Gaussian bandit problem. These examples show how explicit worst-case Bayesian \(\textup{CVaR}\) lower bounds can be obtained and how the resulting bounds compare with the corresponding Bayesian expected-loss benchmarks.

\newgeometry{top=0.75in,left=0.75in,right=0.75in,bottom=0.75in}
Recent work on \(\mathrm{CVaR}\) has largely focused on risk-sensitive criteria that are different from the prior-predictive quantity studied here. In stochastic bandits, Baudry et al.~\cite{baudry2021optimal} evaluate each arm by the \(\mathrm{CVaR}\) of its reward distribution and analyze the corresponding cumulative regret relative to the arm with best \(\mathrm{CVaR}\). Tan and Weng~\cite{tan2023cvar} instead study upper bounds on the \(\mathrm{CVaR}\) of the regret random variable itself. Wang et al.~\cite{wang2023near} consider a cumulative \(\mathrm{CVaR}\)-regret criterion in bandits and reinforcement learning and analyze its expectation, so the \(\mathrm{CVaR}\)-based performance gap is defined within each environment and then aggregated. Du et al.~\cite{du2022provably} study an iterated-\(\mathrm{CVaR}\) objective in episodic reinforcement learning, which yields a dynamic nested risk measure. 

In passive settings, Thomas and Learned-Miller~\cite{thomas2019concentration} study estimation of \(\mathrm{CVaR}\) from i.i.d.\ samples, Soma and Yoshida~\cite{soma2020statistical} study learning with a population objective of the form \(\mathrm{CVaR}\) of the loss under a fixed data-generating distribution, and Mhammedi et al.~\cite{mhammedi2020pacbayes} derive PAC-Bayesian generalization bounds for empirical \(\mathrm{CVaR}\); in each case, \(\mathrm{CVaR}\) is evaluated under a single fixed model dependent law. By contrast, our setting is Bayesian: one first draws a model \(M\) from a prior \(\mu\), then generates the transcript under \(P^{M,\mathrm{Alg}}\), and finally evaluates the tail risk of the resulting loss under this prior-predictive mixture law. 

In this work, we study Bayesian \(\textup{CVaR}\) under the prior-predictive law, which is a natural risk-sensitive analogue of Bayes risk: it captures not only average performance under model uncertainty, but also the severity of unfavorable outcomes induced jointly by the prior and the algorithm. Our goal is to derive explicit information-theoretic lower bounds for this prior-predictive \(\textup{CVaR}\) criterion by instantiating the generalized-Fano framework of Bongole et al.~\cite{bongole2026generalizing}.

\subsection{Contributions}

The main contribution is to make the abstract Bayesian \(\textup{CVaR}\) lower-bound framework of \cite{bongole2026generalizing} concrete and usable in canonical problems. In particular:
\begin{itemize}
    \item We extract a reusable two-point Hellinger--\(\textup{CVaR}\) lower-bound template from the generalized interactive Fano corollaries. 

    \item We instantiate this template in both a passive and an interactive setting, namely two-point Gaussian mean estimation and the two-armed Gaussian bandit.

     \item For both examples, we derive explicit fixed-prior Bayesian \(\textup{CVaR}\) lower bounds that recover the same order as the corresponding Bayes-risk benchmarks, while also retaining dependence on the risk level \(\alpha\). In particular, these immediately imply corresponding worst-case (over priors) Bayesian \(\textup{CVaR}\) lower bounds.

    \item We clarify the relation between Bayesian \(\textup{CVaR}\) and Bayesian expected loss through the inequality \(\textup{CVaR}_\alpha[L]\ge \mathbb{E}[L]\), which provides a natural benchmark for interpreting the derived bounds and highlights what is gained by a tail-sensitive lower-bound analysis.

\end{itemize}

\subsection{Notation}

Random variables are denoted by capital letters, such as \(M\), \(X\), \(Y_t\), and \(H_T\); realizations are denoted by lowercase letters. We write \(\mathbb{E}\) for expectation and \(\mathbf{1}\{\cdot\}\) for indicators. For \(z\in\mathbb{R}\), let
\(
[z]_+ := \max\{z,0\}.
\)

For a model \(m\in\mathcal M\) and an algorithm \(\textup{Alg}\), we write
\(
P^{m,\textup{Alg}}
\)
for the law of the observation or transcript \(X\) induced by running \(\textup{Alg}\) in model \(m\). Equivalently,
\(
X\sim P^{m,\textup{Alg}}
\)
means that \(P^{m,\textup{Alg}}\) is the probability distribution of \(X\) given model \(m\) and algorithm \(\textup{Alg}\). In two-point constructions, we use the shorthand
\(
P_i:=P^{M_i,\textup{Alg}}, \qquad i\in\{1,2\}.
\)

For a prior \(\mu\in\mathcal{P}(\mathcal M)\), where \(\mathcal{P}(\mathcal M)\) is the set of all probability distributions on \(\mathcal M\), we write
\(
\mathbb{E}_{\mu,P^{M,\textup{Alg}}}[\cdot]
\)
for expectation under the joint law obtained by first drawing \(M\sim\mu\) and then \(X\sim P^{M,\textup{Alg}}\). Similarly, for a reference law \(Q\) on the transcript space, we write
\(
\mathbb{E}_{\mu,Q}[\cdot]
\)
when \(M\sim\mu\) and \(X\sim Q\) are independent. 
We write \(\textup{Bern}(p)\) for the Bernoulli distribution with mean \(p\). For probability measures \(P\) and \(Q\) on a common measurable space, \(D_{\textup{KL}}(P\|Q)\) denotes the Kullback--Leibler divergence, and \(D_{H^2}(P\|Q)\) denotes the squared Hellinger divergence, where
\(
D_{\textup{KL}}(P\|Q)
:=
\int \log\!\left(\frac{dP}{dQ}\right)\,dP \text{ and }
D_{H^2}(P\|Q)
:=
1-\int \sqrt{\frac{dP}{d\lambda}\frac{dQ}{d\lambda}}\,d\lambda,
\)
where \(\lambda\) is any common dominating measure. We use the standard inequality
\(
D_{H^2}(P\|Q)\le D_{\textup{KL}}(P\|Q).
\)

For a loss random variable \(L\), \(\textup{CVaR}_\alpha(L)\) denotes the upper-tail Conditional Value-at-Risk at level \(\alpha\in[0,1)\), defined by
\(
\textup{CVaR}_\alpha(L)
:=
\min_{t\in\mathbb{R}}
\left\{
t+\frac{1}{1-\alpha}\,\mathbb{E}[(L-t)_+]
\right\}.
\)

\section{Problem setup}

Interactive statistical decision making (ISDM) provides a common language for passive estimation, sequential decision making, and learning under uncertainty. Following \cite{chen2024assouad}, an ISDM problem is specified by
\begin{itemize}
    \item a model class \(\mathcal{M}\),
    \item an algorithm class \(\mathcal{D}\),
    \item a transcript space \(\mathcal{X}\),
    \item and a nonnegative loss function
    \(
    L:\mathcal{M}\times \mathcal{X}\to \mathbb{R}_+.
    \)
\end{itemize}

For each model \(M\in\mathcal{M}\) and algorithm \(\textup{Alg}\in\mathcal{D}\), the interaction between the algorithm and the model induces a law
\(
X\sim P^{M,\textup{Alg}}
\)
on the transcript \(X\in\mathcal{X}\). Thus:
\begin{itemize}
    \item in a passive estimation problem, if one observes samples \(Y_1,\dots,Y_n\) and then outputs an estimate \(\widehat\theta=\textup{Alg}(Y_1,\dots,Y_n)\), the transcript \(X\) may be taken to include both the batch of observations and the final decision, e.g.
    \(
    X=(Y_1,\dots,Y_n,\widehat\theta);
    \)
    \item in a sequential decision problem, \(X\) may be the full interaction history, including actions, observations, and any terminal decision, for example
    \(
    X=(A_1,Y_1,\dots,A_T,Y_T,\widehat d).
    \)
\end{itemize}
This formulation therefore covers both passive and interactive settings within the same framework.

Under a prior \(\mu\in\mathcal{P}(\mathcal{M})\), the standard Bayesian performance criterion is the Bayes risk
\(
\mathbb{E}_{\mu,P^{M,\textup{Alg}}}[L(M,X)].
\)
Much of the existing lower-bound literature focuses on this expected loss. In this work, the focus is instead on the risk-sensitive criterion
\(
\textup{CVaR}^{\,\mu\otimes P^{M,\mathrm{Alg}}}_\alpha(L) =
\)
\[
\min_{t\in\mathbb{R}}
\left\{
t+\frac{1}{1-\alpha}
\int (L(m,x)-t)_+\,(\mu\otimes P^{M,\mathrm{Alg}})(dm,dx)
\right\}.
\]

 This criterion captures tail-sensitive performance and is especially relevant when rare but costly failures matter.

The main aim in this work is to turn the abstract Bayesian \(\textup{CVaR}\) lower-bound framework of \cite{bongole2026generalizing} into concrete and usable converse results. The emphasis is on showing how its corollaries can be instantiated in canonical problems. In particular, the paper develops a reusable two-point Bayesian \(\textup{CVaR}\) template and applies it in both a passive example and an interactive example, thereby illustrating how explicit \(\textup{CVaR}\) lower bounds can be derived directly from the generalized interactive Fano method.

\subsection{Previous results}
\label{sec:previous-results}

The analysis below builds on the Bayesian \(\textup{CVaR}\) lower-bound framework of \cite{bongole2026generalizing}. We recall the two corollaries that are used in our instantiations.

\begin{corollary}[Bounded-hinge lower bound {\cite[Cor.~2]{bongole2026generalizing}}]
\label{cor:hinge-lower}
Assume \(0\le L(M,X)\le L_{\max}\), and fix \(t\in\mathbb R\). Define
\[
\phi_t(\ell):=\frac{(\ell-t)_+}{L_{\max}}\in[0,1],
\]
\[
b_t
:=
\mathbb{E}_{\mu,Q}[\phi_t(L)]
=
\frac{1}{L_{\max}}\mathbb{E}_{\mu,Q}[(L-t)_+].
\]
Let
\[
B:=\mathbb{E}_{M\sim\mu}D_f(P^{M,\textup{Alg}}\|Q),
\]
and
\[
a_f^{-}(B;b)
:=
\inf\Bigl\{a\in[0,1]:
D_f\big(\textup{Bern}(a)\,\big\|\,\textup{Bern}(b)\big)\le B
\Bigr\}.
\]
Then
\[
\mathbb{E}_{\mu,P^{M,\textup{Alg}}}[(L-t)_+]
\ge
L_{\max}\,a_f^{-}(B;b_t).
\]
\end{corollary}

\begin{corollary}[Bayesian \(\textup{CVaR}\) lower bound {\cite[Cor.~3]{bongole2026generalizing}}]
\label{cor:cvar-lower-bt}
Under the assumptions of Corollary~\ref{cor:hinge-lower},
\[
\textup{CVaR}_\alpha(L)
\ge
\min_{t\in\mathbb R}
\left\{
t+\frac{L_{\max}}{1-\alpha}\,a_f^{-}(B;b_t)
\right\}.
\]
\end{corollary}

Corollaries~\ref{cor:hinge-lower}--\ref{cor:cvar-lower-bt} reduce the derivation of explicit tail-sensitive lower bounds to a concrete instantiation problem: choose a reference law \(Q\), lower bound the corresponding hinge term \(b_t\), and upper bound the divergence budget \(B\). 

\section{Results}
\label{sec:results}
This section presents the reusable two-point Hellinger--\(\textup{CVaR}\) bound and the baseline comparison with Bayes risk that underlie the later examples. Full proofs are deferred to Appendix~\ref{app:proofs}. 

The first result is a reusable two-point Hellinger--\(\textup{CVaR}\) bound that will be instantiated in the examples below. 

\begin{lemma}[Two-point Hellinger--\(\textup{CVaR}\) lower bound]
\label{lem:two-point-hellinger-cvar}
Let \(\mathcal M=\{M_1,M_2\}\), let
\[
\mu=\frac12\delta_{M_1}+\frac12\delta_{M_2},
\]
and fix an algorithm \(\textup{Alg}\). Let
\[
P_i:=P^{M_i,\textup{Alg}},\qquad i=1,2.
\]
Let \(L:\mathcal M\times\mathcal X\to[0,L_{\max}]\) be a bounded loss. Assume that, for some \(C\in[0,2L_{\max}]\),
\[
L(M_1,x)+L(M_2,x)\ge C
\qquad
\text{for all }x\in\mathcal X,
\tag{H1}
\]
and that
\[
D_{H^2}(P_1\|P_2)\le \Gamma_H.
\tag{H2}
\]
Then, for every \(\alpha\in[0,1)\),
\[
\textup{CVaR}_\alpha^{\mu\otimes P^{M,\textup{Alg}}}(L)
\ge\]
\[
\min_{t\in\mathbb R}
\left\{
t+\frac{L_{\max}}{1-\alpha}
\left(
\sqrt{\left(\frac{C/2-t}{L_{\max}}\right)_+}
-\sqrt{\Gamma_H}
\right)_+^2
\right\}.
\]
\end{lemma}

\textit{Proof sketch:}
Apply Corollary~\ref{cor:cvar-lower-bt} with \(f=H^2\) and reference law \(Q=P_2\). Use \((\mathrm{H1})\) to lower bound the hinge benchmark \(b_t\), use \((\mathrm{H2})\) to bound the Hellinger budget, and then apply the Bernoulli Hellinger inversion
\[
a^-_{H^2}(B;b)\ge (\sqrt b-\sqrt{2B})_+^2.
\]
\qed

A simplified version is convenient for the balanced two-point constructions used below.

\begin{corollary}[Balanced-pair Hellinger simplification]
\label{cor:balanced-two-point-hellinger}
Under the assumptions of Lemma~\ref{lem:two-point-hellinger-cvar}, if in addition
\[
L(M_1,x)+L(M_2,x)\ge L_{\max}
\qquad
\text{for all }x\in\mathcal X,
\tag{H3}
\]
then, for every \(\alpha\in[0,1)\),
\[
\textup{CVaR}_\alpha^{\mu\otimes P^{M,\textup{Alg}}}(L)
\ge
\]
\[
\min_{t\in\mathbb R}
\left\{
t+\frac{L_{\max}}{1-\alpha}
\left(
\sqrt{\left(\frac12-\frac{t}{L_{\max}}\right)_+}
-\sqrt{\Gamma_H}
\right)_+^2
\right\}.
\]
\end{corollary}

\textit{Proof sketch:}
Set \(C=L_{\max}\) in Lemma~\ref{lem:two-point-hellinger-cvar}. \qed

The next proposition gives a baseline comparison with Bayesian expected loss.

\begin{proposition}[Bayesian \(\textup{CVaR}\) dominates Bayesian expected loss]
\label{prop:cvar-dominates-bayes-risk}
Let \(L\) be any integrable loss, and let \(\alpha\in[0,1)\). Then
\(
\textup{CVaR}_\alpha(L)\ge \mathbb{E}[L].
\)
Consequently, for any prior \(\mu\),
\[
\textup{CVaR}_\alpha^{\mu\otimes P^{M,\textup{Alg}}}(L)
\ge
\mathbb{E}_{\mu\otimes P^{M,\textup{Alg}}}[L].
\]
In particular, if for some prior \(\mu\) and some quantity \(\psi>0\) one has
\[
\mathbb{E}_{\mu\otimes P^{M,\textup{Alg}}}[L]\ge c\,\psi
\]
for a constant \(c>0\), then
\[
\textup{CVaR}_\alpha^{\mu\otimes P^{M,\textup{Alg}}}(L)\ge c\,\psi.
\]
Moreover,
\[
\sup_{\mu}\textup{CVaR}_\alpha^{\mu\otimes P^{M,\textup{Alg}}}(L)
\ge
\sup_{\mu}\mathbb{E}_{\mu\otimes P^{M,\textup{Alg}}}[L].
\]
\end{proposition}

\section{Examples}
\label{sec:examples}

We now instantiate the two-point Hellinger template in a passive and an interactive Gaussian problem. In both cases, the lower bound is obtained by restricting the full problem class to a carefully chosen two-model subfamily and applying the fixed-prior two-point result on that hard pair. Since a lower bound for one fixed prior immediately implies a lower bound for the supremum over priors, the resulting theorems also yield worst-case (over priors) Bayesian \(\textup{CVaR}\) lower bounds. For readability, we state the main bounds here and defer the full proofs to Appendix~\ref{app:proofs}.

Both examples lead to the same one-dimensional minimization, so we first introduce the closed-form function that appears in the resulting bounds. For \(\alpha\in[0,1)\), define
\[
\Psi_\alpha(\rho):=
\begin{cases}
\displaystyle
\frac12-\frac{\rho^2}{2\alpha},
& 0\le \rho\le \alpha,\ \alpha>0,\\[2ex]
\displaystyle
\frac{(1-\rho)^2}{2(1-\alpha)},
& \alpha<\rho\le 1,\\[2ex]
0,
& \rho\ge 1,
\end{cases}
\]
and, for \(\alpha=0\),
\[
\Psi_0(\rho):=\frac{(1-\rho)_+^2}{2}.
\]
The function \(\Psi_\alpha\) is continuous at the breakpoints \(\rho=\alpha\) and \(\rho=1\); these branch values are obtained by the scalar minimization carried out in Lemma~\ref{lem:scalar-min} in the appendix. We will also use the optimized constant
\[
c_\alpha:=
\begin{cases}
\displaystyle
\frac{2}{27(1-\alpha)},
& 0\le \alpha\le \frac13,\\[2ex]
\displaystyle
\frac{\sqrt{\alpha}}{3\sqrt{3}},
& \frac13\le \alpha<1.
\end{cases}
\]

\subsection{A passive Gaussian mean estimation example}

We begin with a passive Gaussian example, where the transcript consists of the observations together with the final estimator output.

Consider the Gaussian mean estimation problem in which one observes
\[
Y_1,\dots,Y_n \;\;\textup{i.i.d.}\;\; \mathcal N(\theta,1),
\]
with unknown parameter \(\theta\in\mathbb R\), and outputs an estimator
\[
\widehat\theta=\widehat\theta(Y_1,\dots,Y_n).
\]
In the ISDM formulation, the transcript is
\[
X=(Y_1,\dots,Y_n,\widehat\theta).
\]

To apply the two-point Hellinger template, we restrict attention to the hard two-point subfamily
\[
\theta\in\{-\Delta,+\Delta\},
\]
where \(\Delta>0\) will be chosen later. Let
\[
M_+ := M_{+\Delta},\qquad M_-:=M_{-\Delta},
\qquad
\mu:=\frac12\delta_{M_+}+\frac12\delta_{M_-}.
\]
We equip this subfamily with the bounded estimation loss
\[
L(\theta,X):=\min\{|\widehat\theta-\theta|,\;2\Delta\}.
\]

The following theorem gives a fixed-prior Bayesian lower bound on this hard subfamily and hence also a worst-case Bayesian lower bound for the Gaussian mean estimation problem.

\begin{theorem}[Gaussian mean estimation: \(\alpha\)-dependent Bayesian \(\textup{CVaR}\) lower bound]
\label{prop:gaussian-mean-hellinger-cvar}
For every estimator \(\widehat\theta\) and every \(\alpha\in[0,1)\),
\[
\textup{CVaR}_\alpha^{\mu\otimes P^{M,\textup{Alg}}}(L)
\ge
2\Delta\,\Psi_\alpha(2\sqrt{n}\,\Delta).
\]
Equivalently,
\[
\textup{CVaR}_\alpha^{\mu\otimes P^{M,\textup{Alg}}}(L)
\ge\]
\[
\min_{t\in\mathbb R}
\left\{
t+\frac{2\Delta}{1-\alpha}
\left(
\sqrt{\left(\frac{\Delta-t}{2\Delta}\right)_+}
-\sqrt{2n}\,\Delta
\right)_+^2
\right\}.
\]
In particular, optimizing over \(\Delta\) and \(\nu\) gives the worst-case lower bound
\[
\sup_{\nu \in \mathcal{P}(\mathcal{M})}\textup{CVaR}_\alpha^{\nu\otimes P^{M,\textup{Alg}}}(L)
\ge
\sup_{\Delta>0}\textup{CVaR}_\alpha^{\mu\otimes P^{M,\textup{Alg}}}(L)
\ge
\frac{c_\alpha}{\sqrt n}.
\]
\end{theorem}

\textit{Proof sketch:}
The proof has three steps. First, identify the bounded loss level \(L_{\max}=2\Delta\). Second, verify the balanced-pair condition via the triangle inequality:
\[
L(+\Delta,x)+L(-\Delta,x)\ge 2\Delta.
\]
Third, bound the Hellinger divergence between the transcript laws induced by \(M_+\) and \(M_-\) using the Gaussian KL comparison,
\[
D_{H^2}(P^{M_+,\textup{Alg}}\|P^{M_-,\textup{Alg}})
\le
D_{\textup{KL}}(P^{M_+,\textup{Alg}}\|P^{M_-,\textup{Alg}})
=
2n\Delta^2,
\]
and apply Corollary~\ref{cor:balanced-two-point-hellinger}. The closed form then follows from Lemma~\ref{lem:scalar-min}, and optimizing over \(\Delta\) and subsequently over \(\nu\) yields the stated worst-case \(n^{-1/2}\) lower bound. \qed

\subsection{Two-armed Gaussian bandits}

We then turn to an interactive Gaussian example, where the transcript is the full action--observation history.

Consider the two-armed Gaussian bandit problem with unit-variance rewards and horizon \(T\). At each round \(t\), the algorithm \(\textup{Alg}\) chooses an arm \(A_t\in\{1,2\}\) based on the past history and then observes a reward \(Y_t\). The resulting transcript is
\[
H_T=(A_1,Y_1,\dots,A_T,Y_T).
\]

To apply the two-point Hellinger template, we restrict attention to the two-environment subfamily
\[
M_1:=\left(+\frac{g}{2},-\frac{g}{2}\right),
\qquad
M_2:=\left(-\frac{g}{2},+\frac{g}{2}\right),
\]
where \(g>0\) will be chosen later, and equip this subfamily with the symmetric prior
\[
\mu=\frac12\delta_{M_1}+\frac12\delta_{M_2}.
\]
For \(a\in\{1,2\}\), let \(N_a(H_T)\) denote the number of pulls of arm \(a\) up to time \(T\).

The following theorem gives a fixed-prior Bayesian lower bound on this hard two-environment family and hence also a worst-case Bayesian lower bound for the two-armed Gaussian bandit problem.

\begin{theorem}[Two-armed Gaussian bandits: \(\alpha\)-dependent worst-case Bayesian \(\textup{CVaR}\) lower bound]
\label{thm:two-armed-gaussian-hellinger-cvar}
For every algorithm \(\textup{Alg}\) and every \(\alpha\in[0,1)\),
\[
\textup{CVaR}_\alpha^{\mu\otimes P^{M,\textup{Alg}}}(R_T)
\ge
gT\,\Psi_\alpha(g\sqrt{T}).
\]
Equivalently,
\[
\textup{CVaR}_\alpha^{\mu\otimes P^{M,\textup{Alg}}}(R_T)
\ge\]
\[
\min_{t\in\mathbb R}
\left\{
t+\frac{gT}{1-\alpha}
\left(
\sqrt{\left(\frac{gT/2-t}{gT}\right)_+}
-
\frac{g\sqrt{T}}{\sqrt2}
\right)_+^2
\right\},
\]
and optimizing over \(g\) and \(\nu\) gives the worst-case lower bound
\[
\sup_{\nu \in \mathcal{P}(\mathcal{M})}\textup{CVaR}_\alpha^{\nu\otimes P^{M,\textup{Alg}}}(R_T)
\ge
\sup_{g>0}\textup{CVaR}_\alpha^{\mu\otimes P^{M,\textup{Alg}}}(R_T)
\ge
c_\alpha \sqrt{T}.
\]
\end{theorem}

\textit{Proof sketch:}
The proof has the same three-step structure as in the passive example. First, identify the bounded loss level \(L_{\max}=gT\). Second, verify the balanced-pair condition:
\[
R_T(M_1,h)+R_T(M_2,h)=gT
\]
for every transcript \(h\). Third, bound the Hellinger divergence through the Gaussian KL comparison
\[
D_{H^2}(P^{M_1,\textup{Alg}}\|P^{M_2,\textup{Alg}})
\le
D_{\textup{KL}}(P^{M_1,\textup{Alg}}\|P^{M_2,\textup{Alg}})
=
\frac{g^2T}{2},
\]
and apply Corollary~\ref{cor:balanced-two-point-hellinger}. The closed form follows from Lemma~\ref{lem:scalar-min}, and optimizing over \(g\) and subsequently over \(\nu\) yields the stated worst-case \(\sqrt{T}\) lower bound. \qed

\begin{remark}[Interpretation and relation to the examples]
\label{rem:comparison-with-bayes-risk}
Proposition~\ref{prop:cvar-dominates-bayes-risk} shows that any Bayes-risk lower bound immediately yields a \(\textup{CVaR}_\alpha\) lower bound of the same order. Thus, Bayes expected risk provides a natural baseline for interpreting tail-sensitive converse bounds. The value of the generalized-Fano corollaries is that they target the tail directly under the prior-predictive law, provide an instantiation template for both non-interactive and interactive regimes, and can yield more informative statements than expectation alone through explicit dependence on \(\alpha\).

For the Gaussian estimation and two-armed bandit examples, the bounds obtained from the two-point Hellinger template recover the same scaling as the corresponding Bayes-risk benchmarks, namely \(n^{-1/2}\) and \(\sqrt{T}\), respectively. At the same time, unlike the expected Bayes-risk bound, they retain \(\alpha\)-dependence through \(\Psi_\alpha(\cdot)\) and \(c_\alpha\). In this sense, the examples show that the generalized-Fano corollaries are both calibrated to the known Bayesian scaling in canonical two-point problems and capable of capturing tail-sensitive behavior.
\end{remark}
\section{Conclusion and Future Work}

We instantiate the generalized-Fano Bayesian \(\textup{CVaR}\) framework through a reusable two-point Hellinger template. We derive worst-case Bayesian CVaR lower bounds for canonical problems such as Gaussian mean estimation and two-armed Gaussian bandits showing the applicability of the Hellinger template. The resulting bounds recover the known Bayesian expected-risk scaling while also revealing tail risk level \(\alpha\)-dependence, which is invisible to expected-risk lower bounds. Natural next steps are to understand whether multi-point versions of the present Hellinger template can recover sharp Bayesian \(\textup{CVaR}\) scalings in \(K\)-armed bandits and episodic Markov decision processes while preserving explicit dependence on the risk level \(\alpha\), and to compare those lower bounds with matching upper bounds for risk-sensitive algorithms.

\IEEEpeerreviewmaketitle

\bibliographystyle{IEEEtran}
\bibliography{references}
\newpage
\appendix
\section{Proofs}
\label{app:proofs}

\begin{lemma}[Scalar minimization]
\label{lem:scalar-min}
For \(\alpha\in[0,1)\) and \(\rho\ge 0\), define
\[
F_{\alpha,\rho}(x):=
\frac12-x+\frac{(\sqrt{x}-\rho/\sqrt2)_+^2}{1-\alpha},
\qquad x\in[0,1/2].
\]
Then
\[
\inf_{x\in[0,1/2]}F_{\alpha,\rho}(x)=\Psi_\alpha(\rho).
\]
Moreover,
\[
\sup_{\rho\ge 0}\rho\,\Psi_\alpha(\rho)=c_\alpha.
\]
\end{lemma}

\begin{proof}
Write \(x=s^2\), with \(s\in[0,1/\sqrt2]\). Then
\[
F_{\alpha,\rho}(x)
=
\frac12-s^2+\frac{(s-\rho/\sqrt2)_+^2}{1-\alpha}.
\]

We first identify \(\inf_{x\in[0,1/2]}F_{\alpha,\rho}(x)\).

If \(0\le s\le \rho/\sqrt2\), then
\[
F_{\alpha,\rho}(x)=\frac12-s^2,
\]
so the minimum on this region is
\[
\begin{cases}
\frac12-\frac{\rho^2}{2}, & 0\le \rho\le 1,\\[1ex]
0, & \rho\ge 1.
\end{cases}
\]

Now suppose \(\rho/\sqrt2\le s\le 1/\sqrt2\). Then
\[
F_{\alpha,\rho}(x)
=
\frac12-s^2+\frac{(s-\rho/\sqrt2)^2}{1-\alpha}.
\]

If \(\alpha=0\), this becomes
\[
\frac12-\sqrt2\,\rho\,s+\frac{\rho^2}{2},
\]
which is decreasing in \(s\), hence minimized at \(s=1/\sqrt2\), with value
\[
\frac{(1-\rho)^2}{2}
\qquad (0\le \rho\le 1).
\]
Since
\[
\frac{(1-\rho)^2}{2}\le \frac12-\frac{\rho^2}{2}
\qquad (0\le \rho\le 1),
\]
it follows that
\[
\inf_{x\in[0,1/2]}F_{0,\rho}(x)
=
\begin{cases}
\frac{(1-\rho)^2}{2}, & 0\le \rho\le 1,\\[1ex]
0, & \rho\ge 1,
\end{cases}
=
\Psi_0(\rho).
\]

Assume now \(0<\alpha<1\). On \([\rho/\sqrt2,1/\sqrt2]\),
\[
G_{\alpha,\rho}(s):=
\frac12-s^2+\frac{(s-\rho/\sqrt2)^2}{1-\alpha}
\]
is a convex quadratic with
\[
G'_{\alpha,\rho}(s)=\frac{2\alpha}{1-\alpha}s-\frac{\sqrt2\,\rho}{1-\alpha},
\]
so its critical point is
\[
s^\star=\frac{\rho}{\sqrt2\,\alpha}.
\]

If \(0\le \rho\le \alpha\), then \(s^\star\in[\rho/\sqrt2,1/\sqrt2]\), and
\[
G_{\alpha,\rho}(s^\star)=\frac12-\frac{\rho^2}{2\alpha}.
\]
If \(\alpha<\rho\le 1\), then \(s^\star>1/\sqrt2\), so \(G_{\alpha,\rho}\) is decreasing on \([\rho/\sqrt2,1/\sqrt2]\), and its minimum is attained at \(s=1/\sqrt2\), giving
\[
G_{\alpha,\rho}(1/\sqrt2)=\frac{(1-\rho)^2}{2(1-\alpha)}.
\]
Finally, if \(\rho\ge 1\), choosing \(s=1/\sqrt2\) gives value \(0\).

Comparing with the first region,
\[
\frac12-\frac{\rho^2}{2\alpha}\le \frac12-\frac{\rho^2}{2}
\qquad (0\le \rho\le \alpha),
\]
and
\[
\frac{(1-\rho)^2}{2(1-\alpha)}\le \frac12-\frac{\rho^2}{2}
\qquad (\alpha<\rho\le 1).
\]
Hence
\[
\inf_{x\in[0,1/2]}F_{\alpha,\rho}(x)=\Psi_\alpha(\rho)
\]
for all \(\alpha\in[0,1)\).

It remains to maximize \(\rho\Psi_\alpha(\rho)\). Since \(\Psi_\alpha(\rho)=0\) for \(\rho\ge 1\), it suffices to work on \([0,1]\).

If \(\alpha=0\), then
\[
\rho\Psi_0(\rho)=\frac{\rho(1-\rho)^2}{2},\qquad 0\le \rho\le 1,
\]
whose derivative is
\[
\frac{(1-\rho)(1-3\rho)}{2}.
\]
Thus the maximum is attained at \(\rho=1/3\), with value
\[
\frac{(1/3)(2/3)^2}{2}=\frac{2}{27}=c_0.
\]

Now let \(0<\alpha<1\). We have
\[
\rho\Psi_\alpha(\rho)=
\begin{cases}
\displaystyle
\frac{\rho}{2}-\frac{\rho^3}{2\alpha}, & 0\le \rho\le \alpha,\\[2ex]
\displaystyle
\frac{\rho(1-\rho)^2}{2(1-\alpha)}, & \alpha<\rho\le 1.
\end{cases}
\]

On \([0,\alpha]\), the derivative is
\[
\frac12-\frac{3\rho^2}{2\alpha},
\]
so the critical point is \(\rho=\sqrt{\alpha/3}\). Hence the maximum on \([0,\alpha]\) is
\[
\frac{\sqrt{\alpha}}{3\sqrt3}
\quad\text{if }\alpha\ge \frac13,
\]
and otherwise is attained at \(\rho=\alpha\), with value
\[
\frac{\alpha(1-\alpha)}{2}.
\]

On \((\alpha,1]\), the derivative is
\[
\frac{(1-\rho)(1-3\rho)}{2(1-\alpha)},
\]
so the interior critical point is \(\rho=1/3\). Hence the maximum on \((\alpha,1]\) is
\[
\frac{2}{27(1-\alpha)}
\quad\text{if }\alpha\le \frac13,
\]
and otherwise is attained at the boundary \(\rho=\alpha\), with value
\[
\frac{\alpha(1-\alpha)}{2}.
\]

Since the two branches agree at \(\rho=\alpha\), it remains only to compare the boundary value with the interior one. For \(0\le \alpha\le 1/3\),
\[
\frac{\alpha(1-\alpha)}{2}\le \frac{2}{27(1-\alpha)},
\]
and for \(1/3\le \alpha<1\),
\[
\frac{\alpha(1-\alpha)}{2}\le \frac{\sqrt{\alpha}}{3\sqrt3}.
\]
Therefore
\[
\sup_{\rho\ge 0}\rho\,\Psi_\alpha(\rho)=
\begin{cases}
\displaystyle
\frac{2}{27(1-\alpha)}, & 0\le \alpha\le \frac13,\\[2ex]
\displaystyle
\frac{\sqrt{\alpha}}{3\sqrt3}, & \frac13\le \alpha<1,
\end{cases}
\]
which is exactly \(c_\alpha\).
\end{proof}

\begin{proof}[\textbf{Proof of Lemma~\ref{lem:two-point-hellinger-cvar}}]
Apply Corollary~\ref{cor:cvar-lower-bt} with \(f=H^2\) and reference law
\[
Q:=P_2.
\]
For each \(t\in\mathbb R\), define
\[
b_t:=\frac{1}{L_{\max}}\mathbb E_{\mu,Q}[(L-t)_+].
\]
Since \(\mu=\frac12(\delta_{M_1}+\delta_{M_2})\),
\[
b_t
=
\frac{1}{2L_{\max}}
\mathbb E_{P_2}\big[(L(M_1,X)-t)_+ + (L(M_2,X)-t)_+\big].
\]
Fix \(x\in\mathcal X\). By convexity of \(u\mapsto (u-t)_+\),
\[
\frac12\Big((L(M_1,x)-t)_+ + (L(M_2,x)-t)_+\Big)
\ge\]
\[
\left(\frac{L(M_1,x)+L(M_2,x)}{2}-t\right)_+.
\]
Using assumption \((\mathrm{H1})\),
\[
\frac12\Big((L(M_1,x)-t)_+ + (L(M_2,x)-t)_+\Big)
\ge
\left(\frac{C}{2}-t\right)_+.
\]
Taking \(\mathbb E_{P_2}\) and dividing by \(L_{\max}\) yields
\[
b_t\ge \left(\frac{C/2-t}{L_{\max}}\right)_+.
\tag{A.7}
\]

Next,
\[
B
=
\mathbb E_{M\sim\mu}D_{H^2}(P^{M,\textup{Alg}}\|Q)
=\]
\[
\frac12D_{H^2}(P_1\|P_2)+\frac12D_{H^2}(P_2\|P_2)
=
\frac12D_{H^2}(P_1\|P_2).
\]
Using \((\mathrm{H2})\),
\[
B\le \frac{\Gamma_H}{2}.
\tag{A.8}
\]

Corollary~\ref{cor:cvar-lower-bt} gives
\[
\textup{CVaR}_\alpha^{\mu\otimes P^{M,\textup{Alg}}}(L)
\ge
\min_{t\in\mathbb R}
\left\{
t+\frac{L_{\max}}{1-\alpha}\,a^-_{H^2}(B;b_t)
\right\}.
\tag{A.9}
\]
For Bernoulli laws, one has
\[
D_{H^2}\!\bigl(\textup{Bern}(a)\,\big\|\,\textup{Bern}(b)\bigr)
=
1-\sqrt{ab}-\sqrt{(1-a)(1-b)}
=\]
\[
\frac12\Bigl((\sqrt a-\sqrt b)^2+(\sqrt{1-a}-\sqrt{1-b})^2\Bigr).
\]
Hence
\[
D_{H^2}\!\bigl(\textup{Bern}(a)\,\big\|\,\textup{Bern}(b)\bigr)
\ge \frac12(\sqrt a-\sqrt b)^2.
\]
Therefore, if
\[
D_{H^2}\!\bigl(\textup{Bern}(a)\,\big\|\,\textup{Bern}(b)\bigr)\le B,
\]
then
\[
\sqrt a\ge (\sqrt b-\sqrt{2B})_+,
\]
and so
\[
a^-_{H^2}(B;b)\ge (\sqrt b-\sqrt{2B})_+^2.
\tag{A.10}
\]

Combining \((\mathrm{A.7})\), \((\mathrm{A.8})\), and \((\mathrm{A.10})\), we obtain
\[
a^-_{H^2}(B;b_t)
\ge
\left(
\sqrt{\left(\frac{C/2-t}{L_{\max}}\right)_+}
-\sqrt{\Gamma_H}
\right)_+^2.
\]
Substituting this into \((\mathrm{A.9})\) proves the claim.
\end{proof}

\begin{proof}[Proof of Corollary~\ref{cor:balanced-two-point-hellinger}]
Set \(C=L_{\max}\) in Lemma~\ref{lem:two-point-hellinger-cvar}. This gives
\[
\textup{CVaR}_\alpha^{\mu\otimes P^{M,\textup{Alg}}}(L)
\ge
\]
\[
\min_{t\in\mathbb R}
\left\{
t+\frac{L_{\max}}{1-\alpha}
\left(
\sqrt{\left(\frac12-\frac{t}{L_{\max}}\right)_+}
-\sqrt{\Gamma_H}
\right)_+^2
\right\}.
\]
\end{proof}

\begin{proof}[\textbf{Proof of Proposition~\ref{prop:cvar-dominates-bayes-risk}}]
By the dual representation of Expected Shortfall/\(\textup{CVaR}\) \cite{ang2018dual},
\[
\textup{CVaR}_\alpha(L)
=
\sup\Bigl\{
\mathbb{E}[LZ]:
0\le Z\le \tfrac{1}{1-\alpha},\ \mathbb{E}[Z]=1
\Bigr\}.
\]
Since the choice \(Z\equiv 1\) is feasible, it follows that
\[
\textup{CVaR}_\alpha(L)\ge \mathbb{E}[L].
\]
Applying this under the Bayesian law \(\mu\otimes P^{M,\textup{Alg}}\) gives
\[
\textup{CVaR}_\alpha^{\mu\otimes P^{M,\textup{Alg}}}(L)
\ge
\mathbb{E}_{\mu\otimes P^{M,\textup{Alg}}}[L].
\]
The remaining claims follow immediately.
\end{proof}

\begin{proof}[\textbf{Proof of Theorem~\ref{prop:gaussian-mean-hellinger-cvar}}]
We verify the assumptions of Corollary~\ref{cor:balanced-two-point-hellinger}.

\medskip
\noindent
\textit{Step 1: Boundedness and balanced-pair condition.}
By definition,
\[
0\le L(\theta,X)\le 2\Delta.
\]
Hence \(L_{\max}=2\Delta\). Fix any realized estimate \(z\in\mathbb R\). Then
\[
|z-\Delta|+|z+\Delta|\ge 2\Delta
\]
by the triangle inequality, so
\[
\min\{|z-\Delta|,2\Delta\}+\min\{|z+\Delta|,2\Delta\}\ge 2\Delta.
\]
Applying this with \(z=\widehat\theta\) gives
\[
L(+\Delta,X)+L(-\Delta,X)\ge 2\Delta=L_{\max}.
\]

\medskip
\noindent
\medskip
\noindent
\textit{Step 2: Hellinger budget.}
The observation vector \(Y=(Y_1,\dots,Y_n)\) has laws
\[
\mathcal N(\Delta,1)^{\otimes n}
\qquad\text{and}\qquad
\mathcal N(-\Delta,1)^{\otimes n}
\]
under \(M_+\) and \(M_-\), respectively. Since the transcript
\[
X=(Y,\widehat\theta),\qquad \widehat\theta=\textup{Alg}(Y),
\]
is a measurable function of \(Y\), the data-processing inequality gives
\[
D_{H^2}(P^{M_+,\textup{Alg}}\|P^{M_-,\textup{Alg}})\le
D_{\textup{KL}}\bigl(\mathcal N(\Delta,1)^{\otimes n}\,\big\|\,\mathcal N(-\Delta,1)^{\otimes n}\bigr).
\]
Now
\[
D_{\textup{KL}}\bigl(\mathcal N(\Delta,1)^{\otimes n}\,\big\|\,\mathcal N(-\Delta,1)^{\otimes n}\bigr)
=\]
\[
n\,D_{\textup{KL}}(\mathcal N(\Delta,1)\|\mathcal N(-\Delta,1))
=
n\cdot \frac{(2\Delta)^2}{2}
=
2n\Delta^2.
\]
Hence we may take
\[
\Gamma_H:=2n\Delta^2.
\]

\medskip
\noindent
\text{Step 3: Apply Corollary~\ref{cor:balanced-two-point-hellinger}.}
We obtain
\[
\textup{CVaR}_\alpha^{\mu\otimes P^{M,\textup{Alg}}}(L)
\ge
\]
\[
\min_{t\in\mathbb R}
\left\{
t+\frac{2\Delta}{1-\alpha}
\left(
\sqrt{\left(\frac{\Delta-t}{2\Delta}\right)_+}
-\sqrt{2n}\,\Delta
\right)_+^2
\right\}.
\tag{A.11}
\]

\medskip
\noindent
\text{Step 4: Closed form.}
Since \(0\le L\le 2\Delta\), the minimization in \((\mathrm{A.11})\) may be restricted to \(t\in[0,2\Delta]\). Set
\[
x:=\frac12-\frac{t}{2\Delta}\in[0,1/2].
\]
Then \(t=2\Delta(1/2-x)\), and \((\mathrm{A.11})\) becomes
\[
\textup{CVaR}_\alpha^{\mu\otimes P^{M,\textup{Alg}}}(L)
\ge\]
\[
2\Delta
\inf_{x\in[0,1/2]}
\left\{
\frac12-x+\frac{(\sqrt{x}-\sqrt{2n}\,\Delta)_+^2}{1-\alpha}
\right\}.
\]
Since
\(
\sqrt{2n}\,\Delta=\frac{2\sqrt n\,\Delta}{\sqrt2},
\)
Lemma~\ref{lem:scalar-min} with \(\rho=2\sqrt n\,\Delta\) yields
\[
\textup{CVaR}_\alpha^{\mu\otimes P^{M,\textup{Alg}}}(L)
\ge
2\Delta\,\Psi_\alpha(2\sqrt n\,\Delta).
\]

Finally, write \(\rho:=2\sqrt n\,\Delta\), so that \(2\Delta=\rho/\sqrt n\). Then
\[
2\Delta\,\Psi_\alpha(2\sqrt n\,\Delta)=\frac{\rho\,\Psi_\alpha(\rho)}{\sqrt n}.
\]
Maximizing over \(\rho\ge 0\) and using Lemma~\ref{lem:scalar-min} gives
\[
\sup_{\Delta>0}
\textup{CVaR}_\alpha^{\mu\otimes P^{M,\textup{Alg}}}(L)
\ge
\frac{c_\alpha}{\sqrt n}.
\]
\end{proof}

\begin{proof}[\textbf{Proof of Theorem~\ref{thm:two-armed-gaussian-hellinger-cvar}}]
We verify the assumptions of Corollary~\ref{cor:balanced-two-point-hellinger}.

\medskip
\noindent
\text{Step 1: Balanced-pair condition and boundedness.}
Under \(M_1\), arm \(1\) is optimal and arm \(2\) is suboptimal by gap \(g\), so
\[
R_T(M_1,H_T)=g\,N_2(H_T).
\]
Under \(M_2\), arm \(2\) is optimal and arm \(1\) is suboptimal by gap \(g\), so
\[
R_T(M_2,H_T)=g\,N_1(H_T).
\]
Hence, for every transcript \(h\),
\[
R_T(M_1,h)+R_T(M_2,h)
=
g\bigl(N_1(h)+N_2(h)\bigr)
=
gT.
\]
Also,
\[
0\le R_T(M_i,H_T)\le gT,\qquad i=1,2.
\]
Thus \(L(M,H_T):=R_T(M,H_T)\) satisfies the balanced-pair condition with \(L_{\max}=gT\).

\medskip
\noindent
\text{Step 2: Hellinger budget.}
For this two-armed Gaussian construction,
\[
D_{H^2}(P^{M_1,\textup{Alg}}\|P^{M_2,\textup{Alg}})
\le
D_{\textup{KL}}(P^{M_1,\textup{Alg}}\|P^{M_2,\textup{Alg}})
=
\frac{g^2T}{2}.
\]
Hence we may take
\(
\Gamma_H:=\frac{g^2T}{2}.
\)

\medskip
\noindent
\text{Step 3: Apply Corollary~\ref{cor:balanced-two-point-hellinger}.}
We obtain
\[
\textup{CVaR}_\alpha^{\mu\otimes P^{M,\textup{Alg}}}(R_T)
\ge\]
\[
\min_{t\in\mathbb R}
\left\{
t+\frac{gT}{1-\alpha}
\left(
\sqrt{\left(\frac{gT/2-t}{gT}\right)_+}
-
\frac{g\sqrt{T}}{\sqrt2}
\right)_+^2
\right\}.
\tag{A.12}
\]

\medskip
\noindent
\text{Step 4: Closed form.}
Since \(0\le R_T\le gT\), the minimization in \((\mathrm{A.12})\) may be restricted to \(t\in[0,gT]\). Set
\[
x:=\frac12-\frac{t}{gT}\in[0,1/2].
\]
Then \(t=gT(1/2-x)\), and \((\mathrm{A.12})\) becomes
\[
\textup{CVaR}_\alpha^{\mu\otimes P^{M,\textup{Alg}}}(R_T)
\ge\]
\[
gT
\inf_{x\in[0,1/2]}
\left\{
\frac12-x+\frac{(\sqrt{x}-g\sqrt{T}/\sqrt2)_+^2}{1-\alpha}
\right\}.
\]
Lemma~\ref{lem:scalar-min} with \(\rho=g\sqrt T\) yields
\[
\textup{CVaR}_\alpha^{\mu\otimes P^{M,\textup{Alg}}}(R_T)
\ge
gT\,\Psi_\alpha(g\sqrt T).
\]

Finally, write \(\rho:=g\sqrt T\), so that \(gT=\rho\sqrt T\). Then
\[
gT\,\Psi_\alpha(g\sqrt T)=\rho\,\Psi_\alpha(\rho)\,\sqrt T.
\]
Maximizing over \(\rho\ge 0\) and using Lemma~\ref{lem:scalar-min} gives
\[
\sup_{g>0}
\textup{CVaR}_\alpha^{\mu\otimes P^{M,\textup{Alg}}}(R_T)
\ge
c_\alpha\sqrt T.
\]
\end{proof}

\end{document}